\documentclass{article}

\usepackage{arxiv}

\usepackage[utf8]{inputenc} % allow utf-8 input
\usepackage[T1]{fontenc}    % use 8-bit T1 fonts
\usepackage{hyperref}       % hyperlinks
\usepackage{url}            % simple URL typesetting
\usepackage{booktabs}       % professional-quality tables
\usepackage{amsfonts}       % blackboard math symbols
\usepackage{nicefrac}       % compact symbols for 1/2, etc.
\usepackage{microtype}      % microtypography
\usepackage{lipsum}		% Can be removed after putting your text content
\usepackage{graphicx}
\usepackage{natbib}

\usepackage{doi}
\usepackage{amsmath}
\usepackage{multirow}
\usepackage[english]{babel}
\usepackage{amsthm}

\newcommand{\beginsupplement}{%
        \setcounter{table}{0}
        \renewcommand{\thetable}{S\arabic{table}}%
        \setcounter{figure}{0}
        \renewcommand{\thefigure}{S\arabic{figure}}%
     }
\setcitestyle{authoryear, open={(},close={)}}

\title{Adversarial Examples from Dimensional Invariance}

\author{ \href{https://orcid.org/0000-0003-1661-4579}{\includegraphics[scale=0.06]{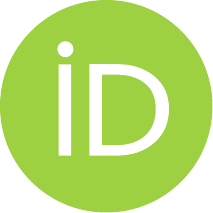}\hspace{1mm}Benjamin L. Badger}\thanks{The author would like to thank Guidehouse for support during the research and writing of this paper.  Code for this work may be found on \url{https://github.com/blbadger/adversarial-theory}} \\
	Guidehouse \\
	1676 International Dr, McLean, VA 22102 \\
	\texttt{bbadger@guidehouse.com} \\
}

\date{}

\renewcommand{\headeright}{Technical Report}
\renewcommand{\undertitle}{Technical Report}
\renewcommand{\shorttitle}{\textit{arXiv} Template}

\hypersetup{
pdftitle={Adversarial Examples from Dimensional Invariance},
pdfsubject={Deep Learning},
pdfauthor={Benjamin L. Badger},
pdfkeywords={Adversarial Examples, Discontinuities},
}

\begin{document}
\maketitle

\begin{abstract}
    Adversarial examples have been found for various deep as well as shallow learning models, and have at various times been suggested to be either fixable model-specific bugs, or else inherent dataset feature, or both. We present theoretical and empirical results to show that adversarial examples are approximate discontinuities resulting from models that specify approximately bijective maps $f: \Bbb R^n \to \Bbb R^m; n \neq m$ over their inputs, and this discontinuity follows from the topological invariance of dimension.  
\end{abstract}

% keywords can be removed
% \keywords{Deep Learning \and Generalization \and Overfitting \and Gradient Descent}

\section{Introduction}

    Adversarial examples were first noted in the context of deep vision models around a decade ago \citep{szegedy2014intriguing}, although similar observations were made for shallow text-based models long before that \citep{wittel2004attacking}.  A number of variations on the adversarial example concept have been described, but for this work we define adversarial examples solely in the context of classification problems, as inputs that are sufficiently `close' by some metric (typically $L^1$ or $L^2$) but which differ drastically in the model's output.

    Adversarial examples have been observed and considered many times previously \citep{shafahi2020adversarial, xiao2019generating, alzantot2018generating}, with conclusions ranging from the idea that adversarial examples are the result of excessive linearity \citep{goodfellow2015explaining} to the notion that adversarial examples implicit in particular datasets regardless of the model used \citep{ilyas2019adversarial}.  In one sense, there is nothing particularly special about adversarial examples being that any classification algorithm with that is unable to sufficiently separate classes will be expected to mistake inputs of one class for another. But the presence of adversarial examples even for models that generalize well (and therefore have presumably learned to separate example classes sufficiently) is less easily accounted for, and the remarkable similarity between an input and its adversarial example further suggests that insufficient class separation alone is not a sufficient explanation for the adversarial example phenomenon.
    
    In this work we take an analytic perspective to shed new light on the phenomenon.  Considering deep learning models to be analytic functions of one ($n$-dimensional) variable $y = f(x)$, $f$ is termed `bijective' if it is one-to-one (injective) and onto (surjective), and are therefore invertible such that for any output $y$ we can recover a unique input $f^{-1}(y) = x$.  

    Deep learning models considered in this work are typical neural network configurations (with fully connected layers) describing continuous, differentiable functions.  It is important to note that even though these functions are strictly continuous, such models may learn to approximate discontinuous functions during training.  Regularization and normalization techniques commonly applied to models during training typically do not prevent the learning of approximate discontinuity.

\section{Deep Learning Classifiers are typically Bijective over their Inputs}
    
    In the context of classification, one typically desires functions that are non-injective maps $f: \Bbb R^n \to {0, 1, ..., m}$ (with many examples of a class being mapped to that class in question). In the case of standard formulations of deep learning, models may indeed be injective due to their requirement for differentiability, and map $f: \Bbb R^n \to \Bbb R^m$.

    Differentiability is an important characteristic because it allows for training via gradient descent, which in turn is remarkably successful in part because it is biased towards generalization when applied to high-dimensional model parameter spaces \citep{badger2022deep}.  A downside is that continuous, differentiable models (particularly those that are high-dimensional) are effectively injective over their inputs because for a finite dataset it is extremely unlikely that any two inputs will yield an identical output $y\in \Bbb R^m$.  This assumption may be empirically checked for trained models on common image classification datasets by computing the minimum pairwise output distance between all dataset inputs indexed by $i, j$ in dataset index $S$ given in Equation (\ref{eq1}).

    \begin{equation}
        d_m = \min || O(x_i, \theta) - O(x_j, \theta) ||_1 \; \forall i, j \in S : \; i \neq j
        \label{eq1}
    \end{equation}

    Table \ref{table1} gives the values obtained for a trained classification model applied to standard datasets.  Note that in every case $d_m > 0$ such that the function learned is effectively injective over the input set. 

    \begin{table}
    \centering
    \begin{tabular}{ |c|c|c|c|c| } 
     \hline
     Dataset & Fashion MNIST & MNIST & CIFAR10 & CIFAR100* \\
     \hline 
     \hline
     $d_m$ & $5.35 \times 10^0$ & $9.2 \times 10^{-1}$ & $8.5 \times 10^{-1}$ & $1.6 \times 10^{-2}$ \\ 
     \hline
    \end{tabular}
    \bigskip
    \caption{Trained classifier $d_m$ per dataset. *with duplicated inputs removed.}
    \label{table1}
    \end{table}
    
    Deep learning classification models are also surjective such that for every possible $y$ (within a given range) one can find at least one $x$ such that $f(x) = y$.  This results directly from the fact that the typical models may be decomposed into functions that are themselves are continuous and unbounded, with the exception of an output layer which is often softmax-transformed for classification tasks.

    Maps that are both injective and surjective (therefore bijective) are invertible, meaning that one can recover a unique input $x$ for any output $y$.  It may seem at first glance that the typical classification model cannot be bijective because it cannot be invertible, given that successive layers of such models are composed of non-invertible functions that do not specify unique inputs given one output.  

    It is almost trivial to see that with a dataset of infinite size there would be no adversarial examples, assuming sufficient model capacity.  This is because the model may simply learn the identity of all possible inputs during training, and assign the correct classification to each one.  Given a finite dataset $X$, however, a function may be 'effectively' invertible if for any output $y$ we can identify the corresponding $x \in X$ using approximate inversion techniques for all dataset and model pairs where $d_m > 0$. 

    To clarify, deep learning models are typically continuous, non-injective functions on all possible inputs.  But they are trained to minimize some objective function on a small subset of all possible inputs, such that all points in input space sufficiently near the input $x_i$ are represented by $x_i$ with respect to the behavior of the function describing the trained model.   

\section{Bijective Maps between different dimensions are Necessarily Discontinuous}

    We have seen that deep learning classification models are effectively bijective functions over their inputs.  In this section we will acquaint the reader with a proof that bijective functions $f: \Bbb R^n \to \Bbb R^m; n \neq m$ cannot be continuous.

    Here `continuous' is a functional property and is defined topologically as follows: in some metric space $(X, d)$ where $f$ maps to another metric space $(Y, d')$, the function $f$ is continuous if and only if for any $\epsilon > 0$,

    \begin{equation}
        \lvert b - a \rvert < \delta \implies \lvert f(b) - f(a) \rvert < \epsilon
    \end{equation}
    
    Where $\delta > 0$ is a distance in metric space $(X, d)$ and $\epsilon$ is a distance in metric space $(Y, d')$.  A discontinuous function is one where the above expression is not true for some pair $(a, b) \in X$ and an everywhere discontinuous function is one in which the above expression is not true for every pair $(a, b) \in X$.  
    
    \topsep=10pt
    \newtheorem{theorem}{Theorem}
    
    \begin{theorem}
    Bijective Maps $f: \Bbb R^n \to \Bbb R^m; n \neq m$ are Necessarily Discontinuous
    \end{theorem}
    
    \begin{proof}
    Consider first the case of $f: \Bbb R^2 \to \Bbb R^1$, and for this case we will abbreviate a proof found in \citep{pierce2012introduction}. Suppose we have arbitrary two points on a two dimensional surface, called $a$ and $b$.  We can connect these points with an arbitrary curve, and now we choose two other points $c$ and $d$ on the surface and connect them with a curve that travels through the curve $ab$ as shown in Figure \ref{fig1}. All four points are mapped to a line, and in particular $a \to a'$, $b\to b'$ etc.

    \begin{figure}[h]
        \centering
        \includegraphics[width=0.75\textwidth]{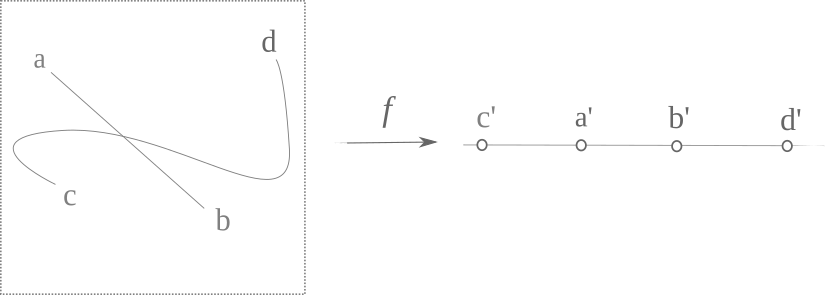}
        \caption{}
        \label{fig1}
    \end{figure}

    Now consider the intersection of $ab$ and $cd$.  This intersection lies between $a'$ and $b'$ because it is on $ab$.  But now note that all other points on $cd$ must lie outside $a'b'$ in order for this to be a one-to-one mapping.  Thus there is some number $\delta > 0$ that exists separating the intersection point from the rest of the mapping of $cd$, and therefore the mapping is not continuous.  To see that it is everywhere discontinuous, observe that any point on the plane may be this intersection point, which maps to a discontinous region of the line.  Therefore a one-to-one and onto mapping of a two dimensional plane to a one dimensional line is nowhere continuous.

    This argument may be generalized to all cases $f: \Bbb R^n \to \Bbb R^m, n > m$ by observing that one can enumerate the dimensions $a \in \{0, 1, ...,  n \}$ as $a_1, a_2, ..., a_n$ and the corresponding dimensions of $b \in \Bbb R^m$ as $b_1, b_2, ..., b_m$ and then simply choosing distinct dimensions $a_m, a_{m+1}$ from $\{n \}$ and the corresponding dimension $b_m$ from $ \{ m \}$ as the relevant dimensions, before applying the arguments detailed above.  

    One can make a slightly different argument to show that $f: \Bbb R^1 \to \Bbb R^2$ cannot in general be continuous if it is bijective. Suppose, without loss of generality, we have just such a continuous bijective function $y = f(x)$ such that two unique points on a line $x_1, x_2 \in x$ are mapped to two unique points $y_1, y_2 \in y$ on a unit square. Then due to the topological definition of the line, between $x_1, x_2$ there exists a cut point $x_c$ such that removal of this point from the line $x$ results in two disconnected lines, one containing $x_1$ and the other $x_2$.  But there is no single point $y_c$ that one can remove from a unit square in order to separate $y_1, y_2$, and therefore $f$ cannot be topologically continuous, and thus not bijective and continuous. As the function $f$ has already been defined to be bijective, it cannot therefore be continuous too.

    This last argument may also be generalized to all cases $f: \Bbb R^n \to \Bbb R^m, n < m$ as above, and therefore no bijective function $f: \Bbb R^n \to \Bbb R^m, n \neq m$ may be continuous.
    
    \end{proof} 

    Deep learning models are described by continuous functions (theoretically speaking, although when implemented on digital hardware they become discontinuous) and classification models are typically continuous and are non-injective over the set of all possible inputs.  But as these models behave as if they were injective over their training data, they may be thought of as approximating functions that are themselves injective.  Combined with the typically surjective nature of these models on their inputs, such models may be considered to be approximately bijective.

    This notion of approximation should not be confused with the $\epsilon$ - neighborhood idea used in analysis but is instead better ideated as a sliding scale: the better a model separates the outputs of all its inputs, the more it resembles (`approximates') a bijective function. 

\section{Empirical Approximate Discontinuity in Models}

    There are many deep learning models that do not map a higher-dimensional space to a lower-dimensional one.  A typical example of this type of model is the autoencoder, which may be described as functions $f:\Bbb R^n \to \Bbb R^m, \; n = m$.  Removing dimensional reduction removes the theoretical necessity of discontinuity, and therefore it is illuminating to see if autoencoders exhibit adversarial examples or not.  

    We focus on adversarial examples generated using the fast gradient sign method \citep{goodfellow2015explaining} given by Equation \ref{eq2} for convenience.  Here $\eta$ is a `learning rate' parameter and $\mathrm{sign}$ is a function $g: \Bbb R \to \{-1, 1 \}$ and $L(O, y)$ is a loss function on the output which is minimized during training.

    \begin{equation}
        x_a = x + \eta * \mathrm{sign} \; \nabla_x L(O(x, \theta), y) 
        \label{eq2}
    \end{equation}

    The fast gradient sign method finds a normalized vector by which the loss function is maximally increased. Consider first an overcomplete autoencoder $f$ defined as $O(x, \theta)$ where $x$ indicates the input and $\theta$ indicates the model configuration.  One measure of the $\epsilon$ - neighborhood expansion $e_a$ for a small adversarial change in $x$ is shown in Equation \ref{eq3}.

    \begin{equation}
        e_a = \frac{||O(x, \theta) - O(x_a, \theta)||_1}{|| x - x_a ||_1}
        \label{eq3}
    \end{equation}

    This may be compared with the $\epsilon$ - neighborhood expansion for a small random change in $x$ given in Equation \ref{eq4}, where $x' = x + \eta * \mathcal{N}(x, \mu, \sigma)$. 

    \begin{equation}
        e_n = \frac{|| O(x, \theta) - O(x', \theta)||_1}{|| x - x' ||_1}
        \label{eq4}
    \end{equation}

    To gain a measure of approximate discontinuity in the direction of the input gradient relative to a random direction, we can simply take the expansion ratio $r = e_a / e_n$ (\ref{eq5}).

    \begin{equation}
        r = \frac{||O(x, \theta) - O(x_a, \theta)||_1 * || x - x' ||_1}{|| O(x, \theta) - O(x', \theta)||_1 * || x - x_a ||_1}
        \label{eq5}
    \end{equation}

    Intuitively this ratio is the expansion of the input space in the direction of the adversarial example to the expansion of input space in a random direction. For any learning rate $\eta$ we can calculate the corresponding expansion ratio $r_{\eta}$.  For a true discontinuity, we expect for $r_{\eta} \to \infty$ as $\eta \to 0$ and indeed this is what is observed when we measured $r_m$ as shown in (\ref{eq5}) using discontinuous models (specifically models with dropout enabled).  For this work, the average $r$ value across hundreds of inputs is measured and plotted as a single line, and multiple experiments with different random initializations for $e_n$ are collected per plot.  We focus on the $\eta$ range that is pertinent to adversarial example generation for these models, which is typically $1 \times 10^{-1} < \eta < 1 \times 10^{-5}$ and note that for $\eta < 1 \times 10^{-7}$ we find numerous numerical instabilities.  

    \begin{figure}[h]
        \centering 
        \includegraphics[width=0.6\textwidth]{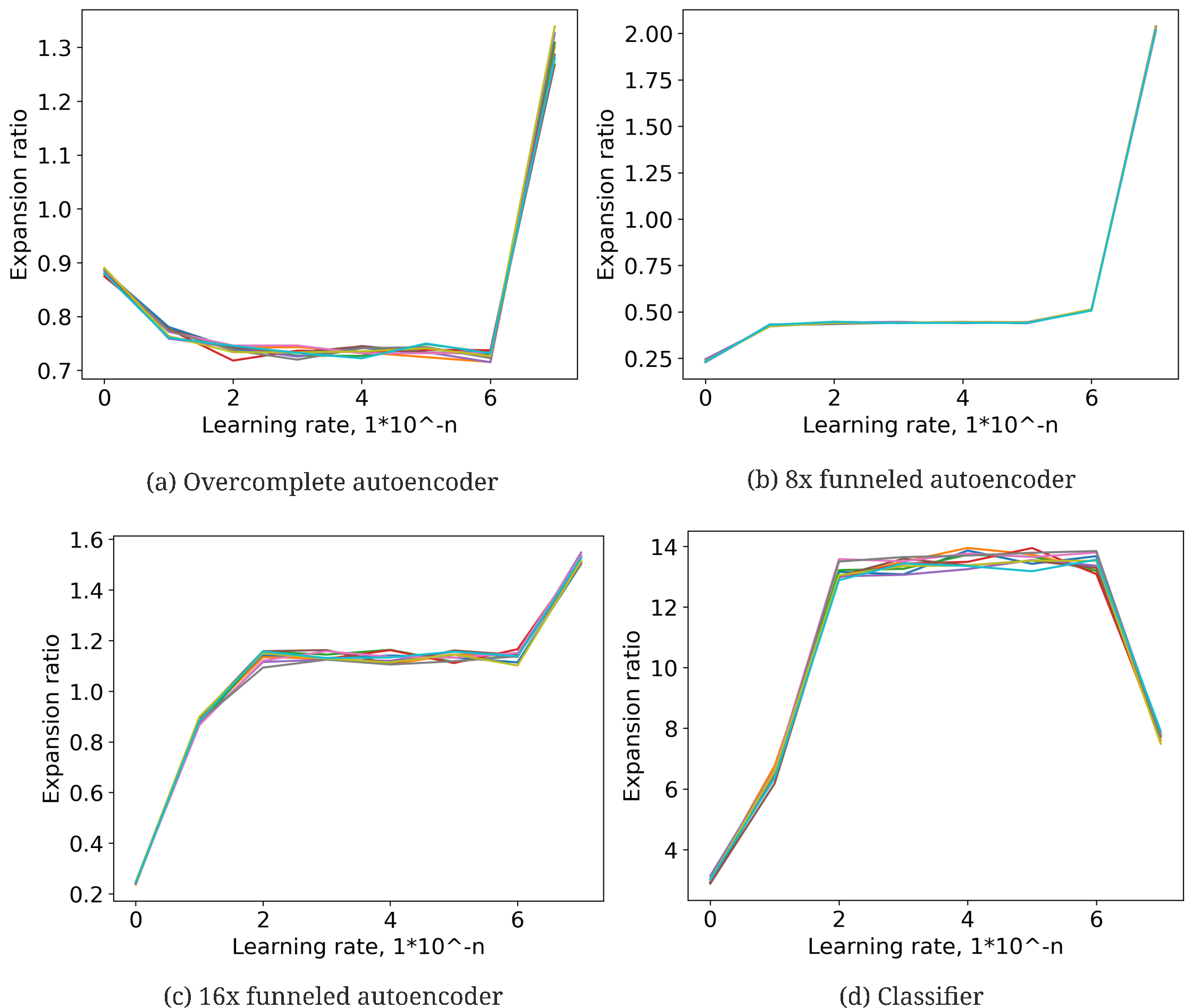}
        \caption{Approximate Discontinuity in Classifiers and GANs but not Autoencoders as $\eta \to 0$.}
        \label{fig2}
    \end{figure}

    An autoencoder represented as a single sub-function is one which maps $\Bbb R^n \to \Bbb R^n$ by definition, but these models may be composed of functions $f: \Bbb R^n \to \Bbb R^m; n > m$ followed by more functions $f: \Bbb R^n \to \Bbb R^m; n < m$.  As shown in Figure \ref{fig1}, the larger the ratio of input to smallest model layer dimension (see the Appendix for more details), the more the approximate discontinuity.  It should be noted that the overcomplete autoencoder presented in Figure \ref{fig1} does not approximate the identity function, and indeed is capable of denoising an input (Figure \ref{figs1}).  

    The process of training a classifier endows the model with the ability to distinguish between various inputs (by definition).  We observe that the training process typically results in a substantial increase in the nearest-neighbors output distance $d_m$ defined in (\ref{eq1}): for example, before training our model exhibits $d_m = 3.0 * 10^{-2}$ which is increased a hundred-fold during training.  We view this as showing that the trained model is two orders of magnitude more `approximately bijective' than the untrained one, and therefore we can expect the model to be more `approximately discontinuous'. Observing the expansion ratio $r_{\eta}$ as $\eta \to 0$ for trained and untrained models, this indeed this is what is found empirically (Figure \ref{figs2}).

\section{GANs moreso than Denoising Diffusion Inversion Autoencoders experience approximate discontinuity}

    The importance of the presence of approximate discontinuity in generative models with high compression ratios may be seen by comparing the presence of such discontinuities in trained generative adversarial networks (GANs) \citep{goodfellow2020generative} and denoising diffusion inversion models \citep{sohl2015deep, ho2020denoising}.  GANs are composed of two models, a generator $f: \Bbb R^n \to \Bbb R^m$ where typically $n << m$ and a discriminator, which maps $f: \Bbb R^m \to \Bbb R$.  Denoising diffusion inversion models are fundamentally denoising autoencoders that map $f: \Bbb R^n \to \Bbb R^n$, although most models are not composed exclusively of overcomplete transformations.  Our implementation of the denoising diffusion inversion model follows that of \citep{ho2020denoising, Wang2023Implementation} except that diffusion step encoding is provided by a continuous value from one input element rather than as a positional encoding over all input elements.

    \begin{figure}[h]
        \centering
        \includegraphics[width=0.6\textwidth]{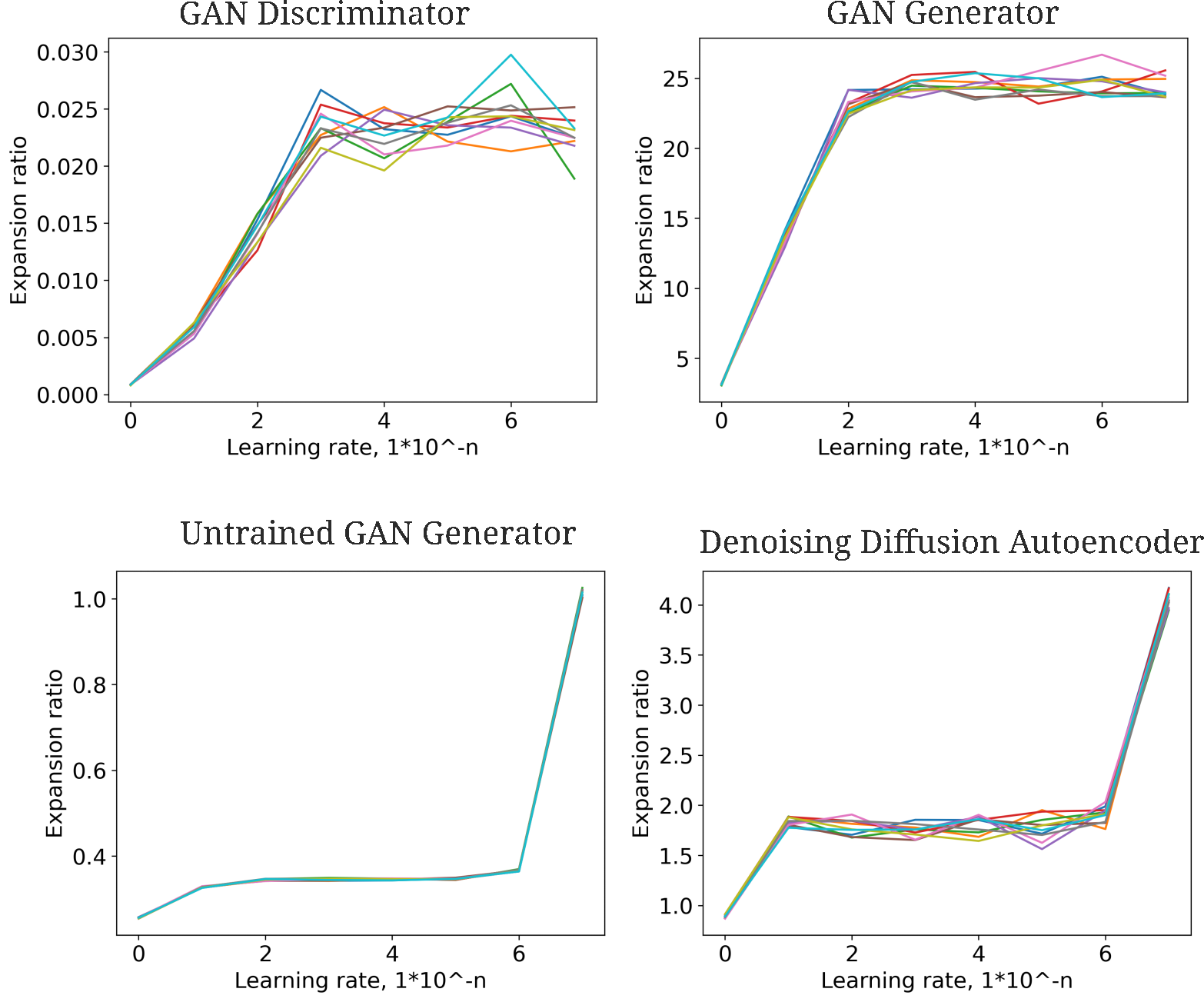}
        \caption{Greater approximate discontinuity in trained GANs than Denoising Diffusion Autoencoders $\eta \to 0$.}
        \label{fig3}
    \end{figure}

    We compared fully connected versions of each model after training on the FashionMNIST dataset and and found that both GAN generator and discriminator exhibit large growth in $r_{\eta}$ values indicating approximate discontinuity, but that the denoising diffusion autoencoder experiences very little growth in $r_{\eta}$ as $\eta \to 0$ (Figure \ref{fig3}).  Indeed the behavior of $r_{\eta}$ for the diffusion autoencoder closely resembles what was observed for the autoencoder in Figure \ref{fig2}.  

    For this version of a GAN, the generator is a function $f: \Bbb R^{100} \to \Bbb R^{768}$ such that there is approximately the same overall compression between smallest versus largest layer in the GAN generator as in the 8x-compressed autoencoder detailed in the last section. The dramatic difference in approximate continuity between these models suggests that the training protocol given to these models is what is responsible for this change, and sure enough we find that the GAN generator experiences very little approximate discontinuity before training (Figure \ref{fig3}).  We conclude that the GAN training protocol itself is responsible for the discontinuity suffered by the generator, which may in part explain why GAN training is so unstable relative to denoising diffusion autoencoder training.

\section{Conclusion}

    In this work we see that deep learning models tend to form approximately discontinuous functions (manifesting as adversarial examples) if they are composed of layers that are of different dimension than subsequent layers.  This observation may explain why the most successful generative models in use today are functions that map $f: \Bbb R^n \to \Bbb R^n$ without severe bottlenecks, being that there is no significant contraction in most current autoregressive language models \citep{vaswani2017attention, workshop2023bloom} and denoising diffusion models.  This work also provides a further explanation for the presence of adversarial examples in models with decreasing width \citep{daniely2020most} and the clear advantage of performing unsupervised pre-training for classification models.
    
\bibliographystyle{unsrtnat}
\bibliography{references}  %%% Uncomment this line and comment out the ``thebibliography'' section below to use the external .bib file (using bibtex) .

\beginsupplement
\section{Appendix}

The hidden layer widths for the 16x funneled autoenocoder are as follows:

$[2000, 1000, 500, 250, 166, 125, 166, 250, 500, 1000, 2000]$ 

and for the 8x funneled autoencoder are 

$[2000, 1000, 500, 250, 500, 1000, 2000]$ 

with no residual connections. Code and trained parameters for all models in this work is available at \url{https://github.com/blbadger/adversarial-theory}

\begin{figure}[h]
        \centering
        \includegraphics[width=0.7\textwidth]{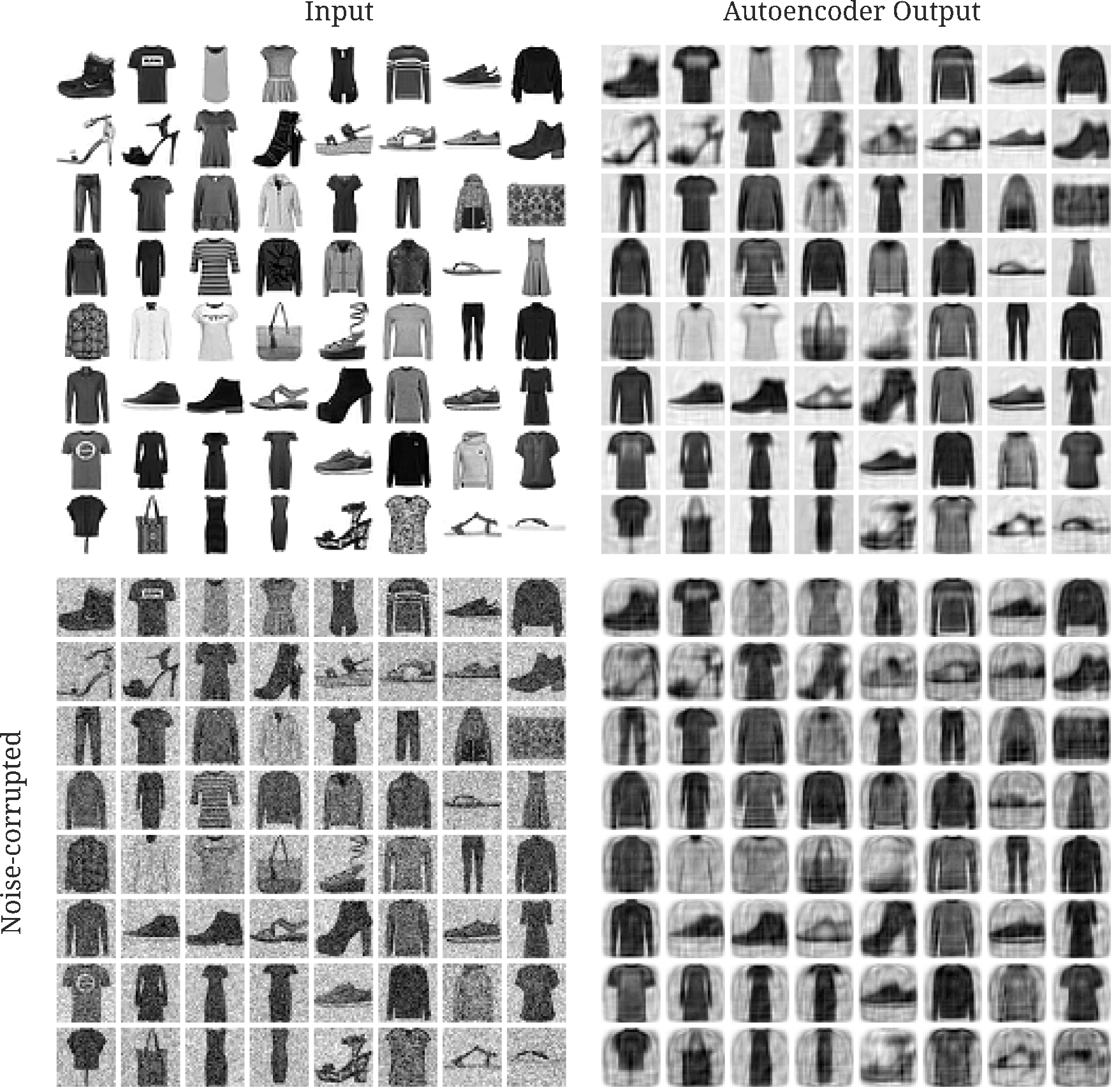}
        \caption{Overcomplete autoencoders are capable of denoising and do not learn the identity function (note the distinctly rounded edges in the autoencoder outputs present in the lower right panel).}
        \label{figs1}
    \end{figure}

\begin{figure}[h]
        \centering
        \includegraphics[width=0.4\textwidth]{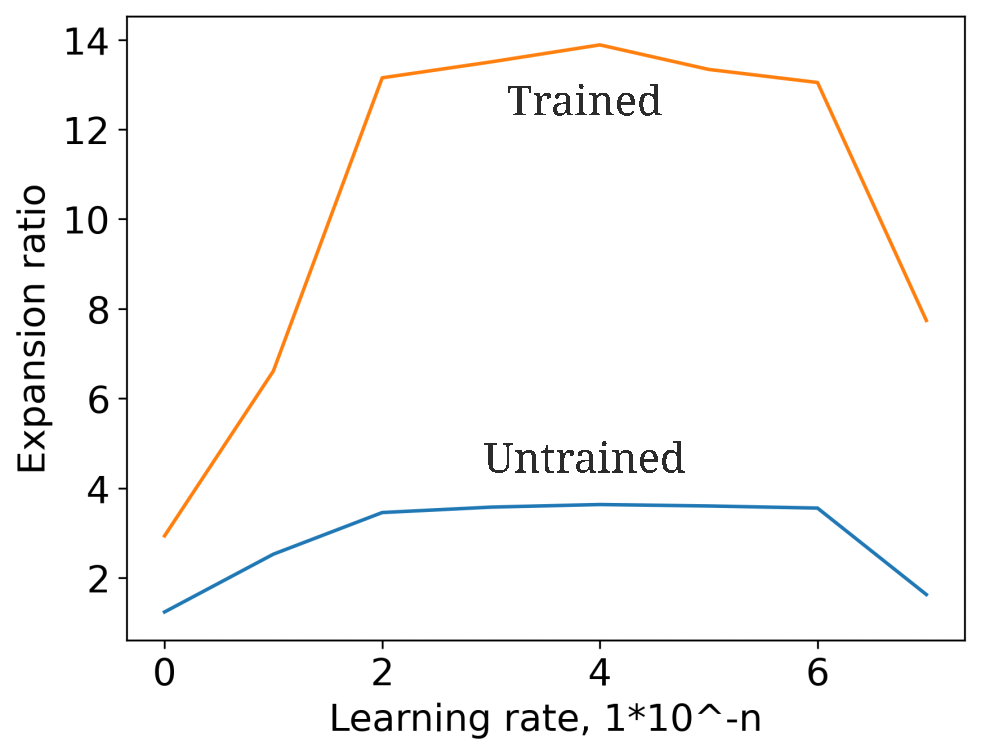}
        \caption{Training increases approximate discontinuity in the classifier.}
        \label{figs2}
    \end{figure}

\end{document}